\documentclass[11pt,a4paper]{article}

\usepackage{amsmath,amssymb,mathrsfs,comment}

\usepackage{algorithm}
\usepackage{hyperref}
\usepackage{algpseudocode}
\usepackage{amsthm}
\usepackage[top=25truemm,bottom=25truemm,left=25truemm,right=25truemm]{geometry}

\newtheorem{theorem}{Theorem}

\newtheorem{lemma}{Lemma}
\newtheorem{definition}{Definition}

\theoremstyle{remark}
\newtheorem{claim}{Claim}
\newtheorem{fact}{Fact}

\newcommand{\F}{\mathbb{F}}
\newcommand{\N}{\mathbb{N}}
\newcommand{\class}{\mathscr{C}}
\newcommand{\Tr}{\mathrm{Tr}}
\newcommand{\poly}{\textup{\text{poly}}}
\newcommand{\E}{\mathrm{E}}
\newcommand{\ol}[1]{\overline{#1}}
\newcommand{\C}{\mathbb{C}}
\newcommand{\fc}[1]{\widehat{#1}}
\newcommand{\R}{\mathbb{R}}
\newcommand{\onen}{\mbox{1}\hspace{-0.25em}\mbox{l}}
\newcommand{\sr}[1]{\textup{\textbf{#1}}}
\newcommand{\qtob}[2]{{#1}_{bin}^{#2}}

\begin{document}

\begin{titlepage}
\title{A Faster Algorithm Enumerating Relevant Features \\over Finite Fields} 

\author{Mikito Nanashima\\Tokyo Institute of Technology\\{\tt nanashima.m.aa@is.c.titech.ac.jp}}
\date{}
\maketitle

\begin{abstract}
We consider the problem of enumerating relevant features hidden in other irrelevant information for multi-labeled data, which is formalized as learning juntas. 

A $k$-junta function is a function which depends on only $k$ coordinates of the input. For relatively small $k$ w.r.t. the input size $n$, learning $k$-junta functions is one of fundamental problems both theoretically and practically in machine learning.
For the last two decades, much effort has been made to design efficient learning algorithms for Boolean junta functions, and some novel techniques have been developed. However, in real world, multi-labeled data seem to be obtained in much more often than binary-labeled one. Thus, it is a natural question whether these techniques can be applied to more general cases about the alphabet size.

In this paper, we expand the Fourier detection techniques for the binary alphabet to any finite field $\mathbb{F}_q$, and give, roughly speaking, an $O(n^{0.8k})$-time learning algorithm for $k$-juntas over $\mathbb{F}_q$. Note that our algorithm is the first non-trivial (i.e., non-brute force) algorithm for such a class even in the case where $q=3$ and we give an affirmative answer to the question posed by Mossel et al.~\cite{Mossel:2004:LFK:1039323.1039330}.

Our algorithm consists of two reductions: (1) from learning juntas to the learning with discrete memoryless errors (LDME) problem which is the extension of the learning with errors (LWE) problems introduced by Regev~\cite{Regev:2005:LLE:1060590.1060603}, and (2) from LDME to the light bulb problem (LBP) introduced by L.Valiant~\cite{Valiant1988}. Since the reduced problem (i.e., LBP) is a kind of binary problem regardless of the alphabet size of the original problem (i.e., learning juntas), we can directly apply the techniques for the binary problem in the previous work.
\end{abstract}
\thispagestyle{empty}
\end{titlepage}

\section{Introduction}
\subsection{Background and Motivation}
In both practical and theoretical senses, it is a fundamental challenge to separate relevant information from irrelevant information in data analysis. In many machine learning settings, collected data may contain many irrelevant features together with relevant features (e.g., DNA sequences and big data), and the efficient techniques for selecting relevant features are widely required.
This problem is captured by learning juntas, which is one of the most challenging and important issues in computational learning theory. Informally, we say an $n$-input function $f:\mathcal{X}^n\to\mathcal{Y}$ is $k$-junta ($k\le n$) iff $f$ depends on only at most $k$ coordinates of the input. Our task is to find the relevant coordinates (i.e., features) of a $k$-junta function $f$, called a target function, from passively collected examples of the form $(x,f(x)) \in\mathcal{X}^n \times \mathcal{Y}$. 

In the special case where the domain of a target function is binary, that is, $\mathcal{X}= \F_2$, the learning junta problem has theoretically important meanings. For $k=O(\log{n})$, learning $k$-junta functions is a special case of learning polynomial-size DNF (disjunctive normal form) formulas and log-depth decision trees, which are also known as notorious open problems in computational learning theory, even in the uniform-distribution model (i.e., examples are distributed uniformly over $\F_2^n$). Therefore, for an affirmative answer to such problems, finding an efficient learning algorithm for log-juntas is inevitable. Despite much effort by many researchers, efficient (i.e., polynomial-time) learning algorithms for log-juntas have not been found.  From the other point of view (i.e., parameterized complexity introduced by~\cite{Downey:1995:FTC:210940.210967}), learning juntas problem can be regarded as a parametrized learning problem for general Boolean functions, and in fact, fixed parameter intractable results have been found in (proper) learning juntas under arbitrary example distribution in~\cite{ARVIND20094928}. However, in the uniform-distribution model, any convincing argument on intractability has not been found until now. For further details about learning juntas, see the survey by Blum~\cite{10.1007/978-3-540-45167-9_54}.

On the positive side, some elegant techniques for learning Boolean juntas have been developed in the uniform-distribution model since the problem was posed in~\cite{ABlum94Relevant,BLUM1997245}.  Obviously, any $k$-junta function can be learned in time $O(n^k)$ with high probability by brute-force search for all $\binom{n}{k} \le n^k$ patterns about relevant coordinates. The first polynomial factor improvement was found by Mossel et al.~\cite{Mossel:2004:LFK:1039323.1039330}, and the running time was reduced to $O(n^{\frac{\omega}{\omega+1}k}) \le O(n^{0.706k})$, where $\omega$ denotes the exponential factor of the running time $O(n^\omega)$ of fast $n\times n$ matrix multiplication with best known bound of $\omega < 2.3728639$ in~\cite{LeGall:2014:PTF:2608628.2608664}. Further improvement has been made by G.Valiant~\cite{Valiant:2015:FCS:2772377.2728167}, and the faster learning algorithm in time $O(n^{\frac{\omega}{4}k}) \le O(n^{0.6k})$  has been developed, which is the best learning algorithm at present. Their contributions are mainly to give a subquadratic algorithm for the light bulb problem which was posed in~\cite{Valiant1988} and a reduction from learning Boolean juntas to the light bulb problem.

In real world, multi-labeled data such as questionnaires or DNA sequences (i.e., (A,T,G,C)) seem to be obtained in much more often than binary-labeled one. Then, it is a natural question whether the techniques for learning Boolean juntas can be modified to more general domains. Although the learning problem for $k$-juntas over the finite alphabet size $q\in\N$ was mentioned as a direction for future work in~\cite{Mossel:2004:LFK:1039323.1039330}, there are much less learnability results in the general case than in the binary case. Obviously, it can be solved in time $O(n^k)$ as in the case $\F_2$. The subsequent work~\cite{5671331} implicitly gave the non-trivial $O(n^{\frac{\omega}{3}k}) \le O(n^{0.8k})$-time algorithm in the case where $q=2^{\ell}$ for some $\ell\in\N$, by reducing the learning problem to $q-1$ learning problems for junta functions of the range $\mathcal{Y} = \F_2$. However, to the best of our knowledge, any non-trivial learning algorithm for juntas over more general domains has not been known, even in the case where $q=3$. In this paper, we investigate the learnability of juntas over arbitrary finite fields, and explicitly give the first non-trivial learning algorithm for such classes.

\subsection{Our Contributions}
Let $\F_q$ be arbitrary finite field of order $q = p^{\ell}$ where $p = char(q)$. In this paper, we focus on $k$-junta functions over $\F_q$ as target functions. Formally, $k$-junta functions are defined as follows.
\begin{definition}
For a function $f:\F_q^n\to\F_q$, we say that a coordinate $i\in\{1,\ldots,n\}$ is relevant if $f(x)\neq f(y)$ for some points of $x,y\in\F_q^n$ such that $x$ and $y$ differ only at the coordinate $i$. For $k\le n$, we say that a function $f$ is $k$-junta if $f$ has at most $k$ relevant coordinates.
\end{definition}

We state the learning junta problem more formally. The learning setting mainly follows the framework of PAC (Probably Approximately Correct) learning which was first introduced by L.Valiant~\cite{Valiant:1984:TL:1968.1972}. The number of relevant coordinates is given in advance by some fixed function $k:\N\to\N$, and a learning algorithm knows the function $k$.  The learning algorithm is given an example oracle $\mathbb{O}(f)$ as the only access to the target function $f:\F_q^n\to\F$. For each access to $\mathbb{O}(f)$, it returns an example $(x,f(x))\in \F_q^n\times\F_q$, where $x$ is selected uniformly at random over $\F_q^n$.

The learning junta problem is formally stated as follows. In this paper, we will use the term ``with high probability (w.h.p. for short)'' to imply with some constant probability.

\begin{quote}
\underline{\textsc{Learning} $k$\textsc{-juntas} (\textsc{over finite field})}

Input: $n,k\in\N$, and an example oracle $\mathbb{O}(f)$ where $f:\F_q^n\to\F_q$ is $k$-junta

Goal: Find all (at most $k$) relevant coordinates w.h.p.
\end{quote}

As described in~\cite{Haussler:1988:EMP:93025.93040}, the failure probability can be reduced to any given $\delta\in(0,1)$ by $O(\ln{\delta^{-1}})$ independent repetitions. The reader may think the above formulation differs from the usual PAC learning model in the sense that the learning algorithm will not output a hypothesis function. However, as described in~\cite{10.1007/978-3-540-45167-9_54,Mossel:2004:LFK:1039323.1039330}, the difficulty of learning juntas comes from the task of finding not what the function is but where the relevant coordinates are. In fact, the above formulation is equivalent to the usual PAC learnability under uniform distribution in learning juntas (within the multiplicative factor of $\poly(n,q^k)$).

In this paper, we will prove the following main result.

\begin{theorem}[main]\label{maintheorem}
For any $\epsilon>0$ and $k=O(\log_q{n})$,
$k$-juntas over any finite field $\F_q$ is learnable in time $n^{\frac{\omega}{3}k+\epsilon}\cdot \poly(n,q^k)$.
\end{theorem}

Our learning algorithm mainly follows the line of work by~\cite{4031391,Valiant:2015:FCS:2772377.2728167} and consists of two reductions that generalize their reductions for the binary domain to any finite field $\F_q$.

In the first step, we reduce the learning juntas problem to another learning problem, learning with discrete memoryless errors (LDME). Simply speaking, the task of LDME is to learn a linear function $\chi_\alpha:\F_q^n\to\F_q$ with $\alpha\in\F_q^n$ under the condition that the label may be corrupted with random noise, where $\chi_\alpha(x) = \alpha_1x_1 +\ldots + \alpha_nx_n$ with arithmetic in $\F_q$. For simplicity, we regard a randomized function as a target function to capture the noise.
\begin{quote}
\underline{\textsc{Learning with Discrete Memoryless Errors}: \textsc{LDME}}

Input: $n,k\in\N$, $\rho \in (0,1]$, and an example oracle $\mathbb{O}(f)$,
\begin{quote}
where $f:\F_q^n\to\F_q$ is randomized. The distribution of the value $f(x)$ is determined by only a value of $\chi_\alpha(x)$ (not $x$ itself), where $1\le|\alpha|\le k$. The target function $f$ is close to $\chi_{\alpha}$ in the sense of correlation as follows:
\[\mathrm{Cor}(f,\chi_\alpha) := |\E_{x,f}[e(f(x))\ol{e(\chi_{\alpha}(x))}]|\ge \rho,\]
where the mapping $e:\F_q\to\C$ is defined by $e(a) = e^{\frac{2\pi i}{p}\Tr(a)}$ for $a\in\F_q$ and $\Tr(a) := \sum_{j = 0}^{\ell-1} a^{p^j} \in \F_p$.
\end{quote}
Goal: Find the coefficients $a\alpha \in \F_q^n$ for some $a\in\F_q\setminus\{0\}$ w.h.p.
\end{quote}

We call the above function $\chi_\alpha$ as a target linear function. The reason why we allow the algorithm to output $a\alpha$ instead of $\alpha$ is that the linear function $\chi_{a\alpha}$ may also have large correlation, that is, $\mathrm{Cor}(f,\chi_{a\alpha})\ge\rho$. 

Let us briefly overview the background of the above problem. LDME, introduced first by~\cite{5671331}, is the extension of the well-known learning with errors problem (LWE) which has been known as one of the most challenging problems in learning theory and even used as a hardness assumption in cryptography (see~\cite{Regev:2005:LLE:1060590.1060603,5497885}). The difference between them is the noise setting. In LWE, the (unknown) distribution of noise is fixed in advance, while in LDME, the distribution is determined for each value of the target linear function, in other words, there exist totally $q$ unknown distributions of the noise. Note that, in addition, we adopt slightly different condition about the closeness between $f$ and $\chi_\alpha$ compared to~\cite{5671331}. In the previous formulation, the given $\rho$ was the lower bound for the agreement probability that $f = \chi_\alpha$. However, in our formulation by correlation, the agreement probability is not always large. For example, even in the case where the subtraction $f-\chi_\alpha$ is close to some constant, our condition about the closeness may hold. 

We first present the reduction from the learning juntas problem to LDME, which is a generalization of the binary case in~\cite{4031391}. The detail will be given in Section~\ref{section:Junta to LWE}.

\begin{theorem}\label{thorem:juntatoLWE}
If there exists a learning algorithm for solving LDME in time $T(n,k,\rho)$, then there exists a learning algorithm for $k$-juntas over $\F_q$ in time $T(n,k,1/q^{k+1})\cdot \poly(n,q^k)$.
\end{theorem}

In the second step, we reduce LDME to the light bulb problem (LBP), which is first introduced by~\cite{Valiant1988} and also a fundamental problem in machine learning and data analysis. Roughly speaking, the task of LBP is to find a correlated pair from the other uncorrelated pairs. The formal definition is as follows:
\begin{quote}
\underline{\textsc{Light Bulb Problem}: \textsc{LBP}}\\
Input: a set $S=\{x^1,\ldots,x^n\}$ of $n$ vectors, and $\rho \in (0,1]$,
\begin{quote}where $x^i\in\{\pm 1\}^d$ for each $i\in[n]$. The instance $S$ contains a single correlated pair $(x^{i^*},x^{j^*})$ satisfying $\langle x^{i^*},x^{j^*} \rangle \ge \rho d$, and the other pairs of vectors are selected independently and uniformly at random.
\end{quote}
Goal: Find indices of the correlated pair $(i^*,j^*)$ w.h.p. 
\end{quote}

It is obvious that LBP is solved in time $O(n^2d)$ by calculating inner products of all pairs. As a breakthrough result, the first subquadratic algorithm for LBP has been found by~\cite{Valiant:2015:FCS:2772377.2728167}. Moreover, in the case where $\rho \ge n^{-\Theta(1)}$, a faster algorithm was presented by~\cite{Karppa:2018:FSA:3233176.3174804}. Other subquadratic algorithms also have been proposed in~\cite{karppa_et_al:LIPIcs:2016:6394,alman:OASIcs:2018:10028}.

\begin{fact}[{\cite[Corollary 2.2]{Karppa:2018:FSA:3233176.3174804}}]\label{fact:KKK18}
For any $0< \epsilon <\omega/3$ and $n^{-\Theta(1)} < \rho <1$, if $d \ge 5\rho^{-\frac{4\omega}{9\epsilon} - \frac{2}{3}}\ln{n}$, then there is a randomized algorithm for solving LBP with probability $1-o(1)$ in time $\tilde{O}(n^{\frac{2\omega}{3}+\epsilon}\rho^{-\frac{8\omega}{9\epsilon} - \frac{4}{3}})$.
\end{fact}

We present the second reduction from LDME to LBP. Note that the reduced problem is a kind of binary problem regardless of the alphabet size of the original problem. The detail will be given in Section~\ref{section:LWE to LBP}.

\begin{theorem}\label{thorem:LWEtoLBP}
Assume that there exist $d\ge\Omega(\frac{\log{N}}{\rho^2})$ and an algorithm for solving LBP of degree $d$ in time $T(N,\rho)$ w.h.p., where $N$ is the number of vectors in LBP. Then for any target linear function $\chi_\alpha:\F_q^n\to\F_q$ $(1\le|\alpha|\le k)$ and any correlation $\rho$, LDME is solved w.h.p. in time  $\poly(n,\rho^{-1})\cdot d\cdot T\left((qn)^{\frac{k}{2}},\frac{\rho}{2q^3}\right)$.
\end{theorem}

In our reduction, the size of data is stretched from $n$ to $O(n^{\frac{k}{2}})$. Thus, the naive quadratic algorithm for LBP does not improve the trivial upper bound on the running time of LDME at all. However, by combining our reductions with the subquadratic algorithm for LBP, we have a non-trivial learnability result which holds for any finite field, and Theorem~\ref{maintheorem} immediately follows from Theorems~\ref{thorem:juntatoLWE} and~\ref{thorem:LWEtoLBP}, and Fact~\ref{fact:KKK18}.

In Theorem~\ref{maintheorem}, the condition that $k=O(\log_q n)$ essentially comes from the condition that $\rho > n^{-\Theta(1)}$ in Fact~\ref{fact:KKK18}. Therefore, by adopting another subquadratic algorithm for LBP that works for any $\rho\in(0,1]$ (e.g.,~\cite{Valiant:2015:FCS:2772377.2728167}), we have a non-trivial learnability result for any $k\le n$. Remark that our reduction and such a subquadratic algorithm also give the non-trivial learning algorithm for LDME, in particular, LWE parameterized by $k$.

\section{Preliminaries}\label{section:Prelimi}
We use $\log$ to denote logarithm of the base 2, and $\ln$ to denote natural logarithm.
For any integer $n$, we define a set $[n] := \{1,2,\ldots,n\}$. Let $\F_q$ be a finite field of order $q = p^{\ell}$ where $p = char(q)$. We define a trace function $\Tr:\F_q\to\F_p$ by $\Tr(a) := \sum_{j=0}^{\ell-1}a^{p^j}$. Note that for any $a,b\in\F_q$, $\Tr(a)+\Tr(b) = \Tr(a+b)$, and $\Tr(\cdot)$ takes on each value in $\F_p$ equally often.

For $\alpha\in\F_q^n$, we define the weight of $\alpha$ by $|\alpha| = |\{i\in[n]:\alpha_i\neq 0\}|$. For $\alpha \neq 0^n$, we also define its initial $init(\alpha)$ by the first non-zero value of $\alpha$, that is,  $init(\alpha) = v$ iff there exists $i\in[n]$ such that $\alpha_i = v$ and $\alpha_j = 0$ for each $1\le j <i$. Note that if $\alpha,\alpha'\in\F_q^n\setminus \{0^n\}$ satisfy $\alpha\neq\alpha'$ and $init(\alpha) = init(\alpha')$, then there is no $c\in\F_q$ such that $\alpha = c\alpha'$ (i.e., $\alpha$ and $\alpha'$ are linearly independent over $\F_q^n$). 

For any $J\subseteq[n]$, we define a subspace $\F_q^J \le \F_q^n$ by $\F_q^J = \{x\in\F_q^n: x_i = 0 \text{ for each } i\in\bar{J}\}$, where $\bar{J} = [n]\setminus J$. For any $\alpha\in\F_q^n$ and $J\subseteq[n]$, we also define $\alpha^J \in \F_q^J$ by $\alpha^J_i = \alpha_i$ if $i\in J$.

For a subset $J\subseteq [n]$, we call a pair $(J,\bar{J})$ a partition of $[n]$. In addition, if $J$ consists of cyclically consecutive $\lceil n/2\rceil$ coordinates, we say that the partition $(J,\bar{J})$ is consecutive. Obviously an index set $[n]$ has exactly $n$ consecutive partitions. Now we introduce the following useful lemma, which says that any subset in $[n]$ is divided into exactly half by at least one consecutive partition of $[n]$. 

\begin{lemma}\label{lemma:concecutive partitions}
For any $\alpha\in\F_q^n$ with $|\alpha| = k$, there exist at least one consecutive partition $(J,\ol{J})$ which satisfies that $|\alpha^J| = \lceil k/2 \rceil$ and $|\alpha^{\bar{J}}| = \lfloor k/2 \rfloor$.
\end{lemma}

\begin{proof}
See Appendix~\ref{proof:concecutive partitions}.
\end{proof}

We use the term ``a truth table'' to denote a table of values of a function over $\F_q$ as in the binary case. 
For any function $f:\F_q^n\to\F_q$ and value $a\in\F_q$, we define a function $af:\F_q^n\to\F_q$ by $af(x) = a\cdot f(x)$. For a subset $J\subseteq[n]$,  we define a restriction $\tau$ on $J$ as a partial assignment to $J$, and we use $f|_\tau:\F_q^{|\bar{J}|}\to\F_q$ to denote the restricted function of which variables are partially assigned $\tau$ on $J$. We use $|\tau|$ to denote the size of a restriction $\tau$, that is, $|\tau| = |J|$.

For a finite set $S$, we write $x\gets_u S$ for a random sampling of $x$ according to the uniform distribution over $S$. 
In the subsequent discussions, we assume the basic facts about probability theory, especially, pairwise independence and the union bound. We will make extensive use of the following tail bound.  

\begin{fact}[Hoeffding inequality~\cite{10.2307/2282952}]
For real values $a,b\in\R$, let $X_1,\ldots,X_m$ be independent and identically distributed random variables with $X_i\in[a,b]$ and $\E[X_i]=\mu$ for each $i\in[m]$. Then for any $\epsilon > 0$, the following inequality holds:
\[\Pr\left[\left|\frac{1}{m}\sum_{i=1}^m X_i -\mu\right|>\epsilon\right] < 2e^{-\frac{2m\epsilon^2}{(b-a)^2}}.\]
\end{fact}

\subsection{Fourier Analysis}
We introduce some basics of Fourier analysis over finite fields. For further details, see~\cite{O'Donnell:2014:ABF:2683783,5671331}.
For each $a\in\F_q$, let $e(a):= e^{\frac{2\pi i}{p} \Tr(a)}\in\C$. For $a,b\in\F_q$, it is easy to see that $e(a+b) = e(a)e(b)$ and $e(-a) = \ol{e(a)}$. For any two functions $f,g:\F_q^n\to \C$, we define their inner product by $\langle f,g \rangle = \E_x[f(x)\ol{g(x)}]$. Then a family $\{e(\chi_\alpha)\}_{\alpha\in\F_q^n}$ of $q^n$ functions forms an orthonormal basis, that is, $\langle e(\chi_\alpha), e(\chi_\beta)\rangle = 1$ if $\alpha = \beta$, otherwise, $\langle e(\chi_\alpha), e(\chi_\beta)\rangle =0$. Therefore, for any function $f:\F_q^n\to\C$ has a unique Fourier expansion form as $f(x) = \sum_{\alpha} \fc{f}(\alpha)e(\chi_\alpha(x))$, where $\fc{f}(\alpha)$ is a Fourier coefficient given by $\fc{f}(\alpha) = \langle f, e(\chi_\alpha)\rangle$.

For a function $f:\F_q^n\to\F_q$ and $\alpha\in\F_q$, we also define its Fourier coefficient on $\alpha$ by $\fc{f}(\alpha) = \langle e(f), e(\chi_\alpha)\rangle$ (we use the same notation as the above). Let us remark that, not as complex-valued functions, $f$ does not always have the unique Fourier form, because the value $f(x) \in \F_q$ is mapped onto $\Tr(f(x)) \in \F_p$ in the definition of $e(\cdot)$, and there exist different functions $f,g:\F_q^n\to\F_q$ which satisfies $\Tr(f) = \Tr(g)$. Our algorithm will extensively use the above analysis, more specifically, it will map the target function $f$ to $\Tr(f)$ and use the Fourier analysis over $\F_p$. However, in the setting of learning juntas, some relevant coordinates for $f$ may turn to be irrelevant for $\Tr(f)$.  This lack of information will be overcome by considering $p^{\ell-1}$ functions $c_1f,\ldots,c_{p^{\ell-1} }f$ simultaneously for distinct elements $c_1,\ldots, c_{p^{\ell-1}} \in\F_q\setminus\{0\}$, which is indicated by the following simple lemma. Note that, for any $c\in\F_q$, we can easily simulate the example oracle $\mathbb{O}(cf)$ from $\mathbb{O}(f)$ by multiplying each label by the value $c$.
\begin{lemma}\label{lemma:simuljunta}
For any function $f:\F_q^n\to\F_q$, distinct elements $c_1,\ldots, c_{p^{\ell-1}} \in\F_q\setminus\{0\}$, and relevant coordinate $i\in[n]$ for $f$, there exists $j\in[p^{\ell-1}]$ such that $i$ is also relevant for $\Tr(c_jf):\F_q^n\to\F_p$.
\end{lemma}

\begin{proof}
By the definition of relevant coordinates, there exists $x,y\in\F_q^n$ such that $x$ and $y$ differ only at the coordinate $i$ and $v := f(x) - f(y) \neq 0$. Since $c_1,\ldots, c_{p^{\ell-1}}$ are distinct and nonzero, the $p^{\ell-1}$ values $c_1v,\ldots, c_{p^{\ell-1}}v$ are also distinct and nonzero. The trace function $\Tr(\cdot)$ takes each value exactly $p^{\ell-1}$ times and $\Tr(0) = 0$, thus there exists $j\in[p^{\ell-1}]$ satisfying $\Tr(c_jv) \neq 0$, which implies
\[\Tr(c_j(f(x) - f(y))) = \Tr(c_jf(x))-\Tr(c_jf(y)) \neq 0.\]
Therefore, $i$ is also relevant for the function $\Tr(c_jf)$.
\end{proof}

We also introduce the following fact which plays a crucial role in learning juntas.
\begin{fact}\label{fact:fcnonzero to relevant}
If a function $f:\F_q^n\to\F_q$ satisfies that $\fc{f}(\alpha)\neq 0$ for some $\alpha\in\F_q^n$, then all coordinates $i\in[n]$ with $\alpha_i \neq 0$ are relevant.
\end{fact}

\begin{proof}
By contraposition. If there exists an irrelevant coordinate $i\in[n]$ such that $\alpha_i \neq 0$, 
\[\fc{f}(\alpha) = \E[e(f(x)-\chi_{\alpha}(x))] = \E[e(f(x)-\chi_{\alpha'}(x))]\cdot \E[e(-\alpha_ix_i)] = 0,\]
where $\alpha' = (\alpha_1,\ldots,\alpha_{i-1},0,\alpha_{i+1},\ldots,\alpha_n)$.
\end{proof}

\subsection{\boldmath{$(a,A)$}-Projection}

We define a notion of $(a,A)$-projection which is a generalization of $A$-projection in $\F_2$ by~\cite{4031391}.

\begin{definition}[\boldmath{$(a,A)$}-projection]
For $f:\F_q^n\to\F_q$, $A\in\F_q^{m\times n}$, and $a\in\F_q$, we define $f^a_A:\F_q^n\to\C$ by
\begin{align*}
f^a_A(x) = \sum_{\alpha:A\alpha = a^m}\fc{af}(\alpha)e(\chi_\alpha(x))
= \begin{cases}
\sum_{\alpha:A\alpha = 1^m}\fc{af}(a\alpha)e(a\chi_\alpha(x)) &(\text{if } a\neq0)\\
1 &(\text{if } a=0)
\end{cases}\end{align*}
\end{definition}

\begin{lemma}\label{lemma:A-projection}
For $A\in\F_q^{m\times n}$ and $a\in\F_q$,
\begin{equation}\label{eq:A pro}
f^a_A(x) = \E_{z\sim\F_q^m}[e(af(x+A^Tz))\ol{e(\chi_{a^m}(z))}].
\end{equation}
Moreover, if an example and its label are given by $(x,b) = (y-A^Tz, f(y)-\sum z_i)$ for $y\gets_u\F_q^n$ and $z\gets_u\F_q^m$, then for any $x\in\F_q^n$, $\E_{b_x}[e(ab_x)] = f^a_A(x)$,
where $b_x$ denotes a random variable according to the distribution of $b$ conditioned on the example $x$.  
\end{lemma}

\begin{proof}
It is essentially the same as the proof in~\cite{4031391}. For completeness, see Appendix~\ref{proof:A-projection}.
\end{proof}

\subsection{Statistical Distance and Character Distance}

For our proofs, we introduce the following two distances about random variables taking values in $\F_q$, which was introduced first in~\cite{Bogdanov:2010:PBP:1958033.1958049}.

\begin{definition}[statistical/character distance]
For random variables $X,X'$ taking values in $\F_q$, we define their statistical distance $SD(X,X')$ by
\[SD(X,X') = \frac{1}{2}\sum_{x\in\F_q}|\Pr[X=x]-\Pr[X'=x]|,\]
and we also define their character distance $CD(X,X')$ by
\[CD(X,X') = \max_{a\in\F_q}|\E[e(aX)]-\E[e(aX')]|.\] 
\end{definition}
In the case where $q$ is not prime, we adopt a different definition for $e(\cdot)$ from one in the original paper~\cite{Bogdanov:2010:PBP:1958033.1958049}. However, it is easily checked that the following fact holds from exactly the same argument.
\begin{fact}[{\cite[Claim 33]{Bogdanov:2010:PBP:1958033.1958049}}]\label{fact:sdcd}
For any random variables $X,X'$ taking values in $\F_q$,
\[CD(X,X') \le 2\cdot SD(X,X')\le \sqrt{q-1}\cdot CD(X,X').\]
In particular, $SD(X,X') = 0$ if and only if $CD(X,X')=0$.
\end{fact}

\section{Reduction from Learning Juntas to LDME}\label{section:Junta to LWE}

In this paper, for simplicity, we assume the following computational model:
\begin{itemize}
	\item  A learning algorithm can uniformly select an element in $\F_q$ with probability 1 in constant steps. In fact, a usual randomized model with binary coins may fail in selecting such random elements with exponentially small probability, but we can deal with this probability as a general error probability (i.e., confidence error). For the same reason, we allow algorithms to flip a biased coin which lands heads up with a rational probability (of the polynomial-time computable denominator).
	\item A learning algorithm with an example oracle $\mathbb{O}(f)$, where $f:\F_q^n \to \F_q$ is $k$-junta, can simulate an oracle $\mathbb{O}(f|_\tau)$ w.r.t. any restriction $\tau$ of the size $|\tau|\le k$. In fact, this simulation is done by taking several examples until getting an example consistent with $\tau$. Since the probability that an example consistent with $\tau$ is sampled is at least $q^{-k}$, the failure probability becomes exponentially small by taking $\poly(q^{k})$ examples. We can also deal with this error probability as a general confidence error, and the additional running time is at most $\poly(n,q^{k})$.
\end{itemize}

\subsection{Overview of the Reduction}\label{section:overview reduction1}

Our learning algorithm (\sr{main1}) has two phases, a checking phase (lines~\ref{line:main1 check} and~\ref{line:main1 halt}) and a detection phase (line~\ref{line:main1 detection}), and repeats them alternately as the MOS algorithm~\cite{Mossel:2004:LFK:1039323.1039330}. The algorithm starts the checking phase with a set $R$ empty. In the following steps, the relevant coordinates found by the algorithm will be put in $R$. In the checking phase, the algorithm verifies whether $R$ contains all relevant coordinates of the target function $f$ by examining that restricted functions $f|_\tau$ are constant for all restrictions $\tau$ on $R$. If $R$ contains all relevant coordinates, then the algorithm outputs $R$ and halts, otherwise moves on to the detection phase. In the detection phase, the algorithm will find at least one relevant coordinate, add them to $R$, and will move on to the checking phase. Since the algorithm finds at least one relevant coordinate in each loop, the number of repetitions is at most $k$.

In the detection phase, we reduce the task of finding relevant coordinates to LDME in the subroutine \sr{addRC} by $(a,A)$-projection. In our reduction, the target linear function $\chi_\alpha$ satisfies that $\fc{cf}(a\alpha)\neq 0$ for some $c,a\in\F_q\setminus\{0\}$. Therefore, if the algorithm for LDME finds $\alpha$ (up to constant factor), then the learning algorithm can find at least one relevant coordinate $i$ such that $\alpha_i\neq 0$ by Fact~\ref{fact:fcnonzero to relevant}.

\subsection{Algorithms and Analysis}

First we introduce two simple subroutines.  For the proofs of their correctness (i.e., Lemmas~\ref{lemma:const} and~\ref{lemma:checkRC}), see Appendix~\ref{proofs in section juntas to lwe}.

Algorithm~\ref{alg:const} checks whether the target function is constant or not by simply examining that the collected examples take the same value. As mentioned in Section~\ref{section:overview reduction1}, we will use this subroutine to determine the end of learning in the checking phase. 

Algorithm~\ref{alg:checkRC} checks whether the given $\alpha\in\F_q^n$ has nonzero entry at an irrelevant coordinate. Our learning algorithm~\sr{main1} may find an undesirable candidate $\alpha$ in the detection phase, thus we must check whether the candidate $\alpha$ consists of only a part of relevant coordinates by this subroutine not to add any irrelevant coordinate to the container $R$ for relevant coordinates.

\begin{algorithm}
\caption{Check Constant (\sr{const})}\label{alg:const}
\begin{algorithmic}[1]
\Require{$n\in\N$, $k\in\N$, $\delta\in(0,1)$, $\mathbb{O}(f)$, where $f:\F_q^n\to\F_q$ is $k$-junta}
\Ensure{$a\in \F_q$ if $f \equiv a$ (constant), otherwise $\bot$} 
\State{$m := \lceil q^k\ln \frac{2}{\delta} \rceil$}
\State{$(x^{(1)},b^{(1)}),\ldots, (x^{(m)},b^{(m)}) \gets \mathbb{O}(f)$}
\State{\sr{if} $b^{(i)} = a$ for each $i\in[m]$ \sr{return}($a$) \sr{otherwise} \sr{return}($\bot$)}
\end{algorithmic}
\end{algorithm}

\begin{algorithm}
\caption{Check Relevant Coordinates (\sr{checkRC})}\label{alg:checkRC}
\begin{algorithmic}[1]
\Require{$n\in\N$, $k\in\N$, $\alpha\in\F_q^n$, $\delta\in(0,1)$, $\mathbb{O}(f)$, where $f:\F_q^n\to\F_q$ is $k$-junta}
\Ensure{$\fc{f}(\alpha) \neq 0 \Rightarrow$ \sr{true}; $\alpha_i \neq 0$ for some irrelevant $i\in[n] \Rightarrow$ \sr{false}} 
\State{$m := \lceil 2q^{2k}\ln \frac{p}{\delta} \rceil$}
\ForAll{$a\in\F_p$}
\State{$(x^{(1)},b^{(1)}),\ldots, (x^{(m)},b^{(m)}) \gets \mathbb{O}(f)$}
\State{\sr{if} $\sum_i\onen\{\Tr(b^{(i)}-\chi_\alpha(x^{(i)}))=a\} \ge (\frac{1}{p}+\frac{1}{2q^{k}})m$ \sr{then} \sr{return}(\sr{true})}\label{line:checkRC condtion}
\EndFor
\State{\sr{return}(\sr{false})}
\end{algorithmic}
\end{algorithm}

\begin{lemma}\label{lemma:const}
For any input $(n,k,\delta,\mathbb{O}(f))$, \sr{const} outputs $a\in\F_q$ if $f\equiv a$, otherwise $\bot$ with probability at least $1-\delta$.
\end{lemma}

\begin{proof}
See Appendix~\ref{proofs for subroutines:const}.
\end{proof}

\begin{lemma}\label{lemma:checkRC}
For any input $(n,k,\alpha,\delta,\mathbb{O}(f))$, if $\fc{f}(\alpha) \neq 0$, then \sr{checkRC} outputs \sr{true} with probability at least $1-\delta$. Otherwise if $\alpha_i \neq 0$ for an irrelevant coordinate $i\in[n]$, \sr{checkRC} outputs \sr{false} with probability at least $1-\delta$.
\end{lemma}

In general, there is a case where $\fc{f}(\alpha) = 0$ and all $i\in[n]$ satisfying $\alpha_i \neq 0$ are relevant. In the above lemma, we do not say anything about such a case.

\begin{proof}
See Appendix~\ref{proofs for subroutines:checkRC}.
\end{proof}

Algorithm~\ref{alg:addRC} is a core part of our reduction, which reduces the task of finding candidates for relevant coordinates to LDME, checks whether the candidates are indeed relevant, and returns them to the main algorithm. Let \sr{LDME}$(n,k,\rho)$ be the learning algorithm for LDME.

\begin{algorithm}
\caption{Add Relevant Coordinates (\sr{addRC})}\label{alg:addRC}
\begin{algorithmic}[1]
\Require{$n\in\N$, $k\in\N$, $\delta\in(0,1)$, $R\subseteq[n]$, $\mathbb{O}(f)$, where $f:\F_q^n\to\F_q$ is $k$-junta}
\Ensure{$R\cup R'$ where $R'\subseteq[n]$ is a subset of relevant coordinates not contained in $R$}
\State{Select $p^{\ell-1}$ distinct elements $c_1,\ldots,c_{p^{\ell-1}} \in \F_q\setminus\{0\}$}
\ForAll{restrictions $\tau$ on $R$ and $j\in[p^{\ell-1}]$}\label{AddRC:repeat all restrictions}
\For{$M := \lceil q^{k+2} \ln\frac{4}{\delta}\rceil$ times}
\State{$A\gets_u \F_q^{(k+1)\times (n-|R|)}$}\label{AddRC:select A}
\ForAll{$a \gets_u \F_q\setminus\{0\}$}\label{AddRC:select a}
\State{execute $\alpha \gets$\sr{LDME}($n-|R|,k, 1/q^{k+1}$) with confidence $\delta/4$ (by repetition)}\label{AddRC:try LDME}
\Statex{\qquad\qquad\qquad where the example oracle is simulated as follows:}
\Statex{\begin{quote}
\qquad\qquad1: get an example $(x,b) \gets \mathbb{O}(c_jf|_\tau)$\\
\qquad\qquad2: select $p \gets_u \F_q^{k+1}$\\
\qquad\qquad3: $(x',b') := (x-A^Tp, a\cdot(b-\sum_j p_j))$ and return $(x',b')$
\end{quote}}\label{line:addRC Aa projection}
\State{\sr{if} \sr{checkRC}$(n-|R|,k,a'\alpha,\frac{\delta}{2Mq^{k+3}},\mathbb{O}(c_jf|_\tau))$ for some $a'\in\F_q\setminus\{0\}$}\label{line:addRC:check candidate}
\State{\qquad \sr{then} add all $i$ s.t. $\alpha_i \neq 0$ to $R$ and \sr{return}($R$)}\label{line:addRC add RV}
\EndFor
\EndFor
\EndFor
\end{algorithmic}
\end{algorithm}

We briefly explain how the subroutine \sr{addRC} works. The details will be addressed in Lemma~\ref{lemma:addRC} and Appendix~\ref{section:proof of lemma addRC}.

If the given set $R$ does not contain all relevant coordinates, then for some restriction $\tau$ on $R$, the restricted function $f|_{\tau}$ is not constant. By Lemma~\ref{lemma:simuljunta}, there exists an element $c_j$ such that $\Tr(c_jf|_{\tau})$ is also non-constant. This subroutine works for such a restriction $\tau$ and an element $c_j$, and finds new relevant coordinates for the function $\Tr(c_jf|_{\tau})$. In fact, the subroutine tries all (at most $q^k$) restrictions on $R$ and elements $c_j$ (line~\ref{AddRC:repeat all restrictions}).

Let $n' := n-|R|$. For the (non-constant) restricted function $c_jf|_{\tau}:\F_q^{n'}\to \F_q$, \sr{addRC} repeats the following process: (1) selects a matrix $A\in\F_q^{(k+1)\times n'}$ at random (line~\ref{AddRC:select A}), (2) selects a value $a\in\F_q$ (line~\ref{AddRC:select a}), and (3) executes \sr{LDME} with the example oracle simulated as in Lemma~\ref{lemma:A-projection} w.r.t the selected $A$ and $a$ (line~\ref{AddRC:try LDME}).

Let $g=c_jf|_{\tau}$. Since the function $\Tr(g):\F_q^{n'}\to\F_p$ is not constant, it has a non-zero coefficient $\fc{g}(\alpha) \neq 0$ of $|\alpha|>0$, which means that $g$ has some correlation with the linear function $\chi_\alpha$. In fact $g$ may have correlation with other linear functions, but the number of such linear functions is small because $g$ is also $k$-junta. Simply speaking, the role of $A$ is to filter out some of these correlations on simulated examples, and we can show that the non-negligible fraction of $A$'s remove all the correlations except for the linear function $\chi_\alpha$ (Claim~\ref{claim:A projected size is small} in Appendix~\ref{section:proof of lemma addRC}). In other words, the simulated examples depend on only $\chi_\alpha$, and it is just an instance of LDME. While, the role of $a$ is to enhance the correlation with the target linear function, and for a good choice of $a$, the correlation is bounded below by $1/q^{k+1}$ (Claim~\ref{claim:correlation is large} in Appendix~\ref{section:proof of lemma addRC}).

If the algorithm \sr{LDME} finds $c\alpha$ for some constant $c\in\F_q\setminus\{0\}$, by Fact~\ref{fact:fcnonzero to relevant} and the fact that $\fc{g}(\alpha) \neq 0$, all coordinates taking non-zero values are relevant for $g=c_jf|_{\tau}$. Moreover, they are also relevant for $f|_\tau$ because the algorithm selected non-zero $c_j$. Therefore, we can reduce the task of finding relevant coordinates to LDME of the correlation bound $\rho = 1/q^{k+1}$.

In fact, for a bad choices of $A$ and $a$, the algorithm may find undesirable candidates $\alpha$. Not to add irrelevant coordinates to $R$ in such a case, \sr{addRC} executes \sr{checkRC} for any candidate found by \sr{LDME} (line~\ref{line:addRC:check candidate}).

\begin{lemma}\label{lemma:addRC}
If the algorithm \sr{LDME} solves LDME in time $T(n,k,\rho)$ w.h.p. and $R$ does not contain all relevant coordinates, then the subroutine \sr{addRC} adds at least one relevant coordinate to $R$ with probability at least $1-\delta$, and its running time is bounded above by $T(n,k,1/q^{k+1})\cdot \poly(n,k,\ln{\delta^{-1}})$.
\end{lemma}

\begin{proof}
The outline is shown in the above. For the complete proof, see Appendix~\ref{section:proof of lemma addRC}.
\end{proof}

Algorithm~\ref{alg:main} is our learning algorithm. Now we prove its learnability by Lemma~\ref{lemma:addRC}. Theorem~\ref{thorem:juntatoLWE} immediately follows from Lemma~\ref{lemma:Main} by substituting some constant for $\delta$.

\begin{algorithm}
\caption{\sr{main1}}\label{alg:main}
\begin{algorithmic}[1]
\Require{$n\in\N$, $k\in\N$, $\delta\in(0,1)$, $\mathbb{O}(f)$, where $f:\F_q^n\to\F_q$ is $k$-junta}
\Ensure{$R\subseteq [n]$ consisting of all relevant coordinates}
\State{$R:=\emptyset$}
\Loop
\State{\sr{if} $|R|>k$ \sr{then} halt and output ``\sr{error}''}
\If{$\bot \not\gets $\sr{const}$(n-|R|,k,\delta/((k+1)q^k+k),\mathbb{O}(f|_\tau))$ for all restrictions $\tau$ on $R$}\label{line:main1 check}
\State{Halt and output $R$}\label{line:main1 halt}
\Else
\State{$R \gets $\sr{addRC}$(n,k,\delta/((k+1)q^k+k),R,\mathbb{O}(f))$}\label{line:main1 detection}
\EndIf
\EndLoop
\end{algorithmic}
\end{algorithm}

\begin{lemma}\label{lemma:Main}
If the algorithm \sr{LDME} solves LDME in time $T(n,k,\rho)$ w.h.p., then the algorithm \sr{main1} outputs all relevant coordinates for any $k$-junta function $f:\F_q^n\to\F_q$ with probability at least $1-\delta$, and its running time is bounded above by $T(n,k,1/q^{k+1})\cdot \poly(n,k,\ln{\delta^{-1}})$.
\end{lemma}

\begin{proof}
First we show that the algorithm halts at most $k+1$ loops assuming that all subroutines succeed. If $R$ contains all relevant coordinates, then for all restrictions $\tau$ on $R$, the restricted functions $f|_\tau$ must be constant, thus the algorithm halts and outputs $R$ in line~\ref{line:main1 halt}. On the other hand, if $R$ does not contain some relevant coordinates, \sr{addRC} adds at least one relevant coordinate to $R$ by Lemma~\ref{lemma:addRC}. Since $f$ has at most $k$ relevant coordinates, \sr{addRC} is executed at most $k$ times, and the main loop is repeated at most $k+1$ times.

In fact, the algorithm may fail in executing \sr{const} and \sr{addRC}. The number of the executions is at most $(k+1)q^k + k$. Thus if we set their confidence parameter as $\delta/((k+1)q^k + k)$, then by the union bound, the total failure probability is bounded above by $\delta$. By Lemma~\ref{lemma:addRC}, the total running time is at most
\begin{multline*}
(k+1)q^k\cdot O\left(n\cdot q^k\ln \frac{(k+1)q^k + k}{\delta} \right) + k\cdot T(n,k,1/q^{k+1})\cdot \poly\left(n,q^k,\ln{\frac{(k+1)q^k + k}{\delta}}\right) \\ = T(n,k,1/q^{k+1})\cdot \poly(n,q^k,\ln{\delta^{-1}}).
\end{multline*}
\end{proof}

\section{Reduction from LDME to LBP}\label{section:LWE to LBP}

First we introduce the following simple lemmas and their corollaries as observations of LDME.

\begin{lemma}\label{lemma:how far the correlated pair}
Let $X$ be a random variable taking values in $\F_q$. For $0\le\rho\le1$, if $|\E[e(X)]|\ge \rho$, then there exists $a\in\F_q$ such that
$\Pr[X=a] \ge \frac{1}{q} + \frac{\rho}{q^2}$.
\end{lemma}

\begin{proof}
See Appendix~\ref{proofs of claims:how far the cor}.
\end{proof}

\begin{lemma}\label{lemma:LI uniform}
Let $\alpha,\beta\in\F_q^n\setminus\{0^n\}$ and $X$ be a random variable taking values in $\F_q$. If the distribution of $X$ is determined by only the value of $\chi_\alpha(x)$ where $x\gets_u \F_q^n$, and $\alpha$ and $\beta$ are linearly independent over $\F_q^n$, then for all $a\in\F_q$,
$\Pr_{x,X}[X-\chi_{\beta}(x) = a] = \frac{1}{q}$.
\end{lemma}

\begin{proof}
See Appendix~\ref{proofs of claims:li uniform}.
\end{proof}

As a corollary, we have the following facts about LDME. Let $\alpha,\beta\in\F_q^n\setminus\{0^n\}$, $\chi_\alpha$ be a target linear function, and $f$ be the target (randomized) function, that is, $\mathrm{Cor}(f,\chi_\alpha) \ge \rho$. If $\beta = \alpha$, then by Lemma~\ref{lemma:how far the correlated pair}, there exists some value $a\in\F_q$ such that $\Pr[f(x) -\chi_{\beta}(x) = a] \ge 1/q + \rho/q^2$. On the other hand, if $\beta$ and $\alpha$ are linearly independent, then by Lemma~\ref{lemma:LI uniform}, $\Pr[f(x) -\chi_{\beta}(x) = a] = 1/q$ for all $a\in\F_q$. We essentially use the difference in our reduction. Note that we do not say anything about the case where $\beta \neq \alpha$ but they are linearly dependent (i.e., $\beta=c\alpha$ for some $c\in\F_q\setminus\{0,1\}$).

\subsection{Overview of the Reduction}\label{section:overview LWEtoLBP}

Our learning algorithm is Algorithm~\ref{alg:LDME} (\sr{main2}) and the main idea is similar to the split-and-list idea in previous work~\cite{Valiant:2015:FCS:2772377.2728167,Karppa:2018:FSA:3233176.3174804}. Let $\alpha\in\F_q^n$ be the coefficients of a target linear function with $|\alpha|\le k$. First we select a consecutive partition that divides the nonzero entries of $\alpha$ into half by brute-force search (line~\ref{line:LDME: select cp}), then list the values of linear functions $\chi_\beta$ of weight $1\le|\beta|\le k/2$ where $\beta$ is contained in either $\beta\in\F_q^{J}$ or $\beta\in\F_q^{\bar{J}}$ (lines~\ref{line:LDME: make instance}:1--4). Not to contain linearly dependent linear functions, we fix an initial value of the coefficient vector for each partition. Since there are at most $(q-1)^2$ patterns about the initial values, we can easily guess the pair of initial values consistent with $\alpha^J$ and $\alpha^{\bar{J}}$.

As the above, we stretch a noisy example to $O(n^{\frac{k}{2}})$ entries taking values in $\F_q$. Then, we translate the stretched data into an instance of LBP, that is, a $\{\pm 1\}$-valued instance. We can observe the following three facts. First, each entry takes values uniformly over $\F_q$. Second, the pair of entries corresponding to $\alpha$ (we may call it a target pair) has some correlation in the sense that they take a certain value $a\in\F_q$ with relatively high probability, where we refer to such a value $a$ as a concentrated value. Finally, other pairs are distributed pairwise independently. 

Now we translate each entry $a\in\F_q$ into $1$ or $-1$ as follows: (1) For the case where $a$ is concentrated, we change the entry to $1$ (line~\ref{line:LDME: make instance}:5), (2) for the case where $a$ is not concentrated, we flip a biased coin with the head probability $q/(2(q-1))$, and if it comes up with head, then we change the entry to $-1$, otherwise to $1$ (line~\ref{line:LDME: make instance}:6). Because each entry is uniformly distributed, the probability that the entry is changed to $-1$ is exactly $\frac{q-1}{q} \cdot \frac{q}{2(q-1)} = \frac{1}{2}$, that is, uniformly distributed over $\{\pm 1\}$. Moreover, by pairwise independence, all pairs except for the target pair are also independently distributed. On the other hand, in the target pair, the correlation remains even in resulting binary instance. In other words, the reduced instance is just the one of LBP.

\subsection{Algorithms and Analysis}

First, we introduce the following simple subroutine Algorithm~\ref{alg:checkCor}, which checks whether a candidate linear function found in the main routine is indeed a target linear function or not. In fact, it can be also implemented by the standard empirical estimation of the correlation. The merit of our implementation by using the conditions in Lemmas~\ref{lemma:how far the correlated pair} and~\ref{lemma:LI uniform} is simply to avoid calculations of complex numbers.

\begin{algorithm}
\caption{Check Correlation (\sr{checkCor})}\label{alg:checkCor}
\begin{algorithmic}[1]
\Require{$n\in\N$, $\rho\in(0,1)$, $\gamma\in\F_q^n$, $\delta\in(0,1)$, $\mathbb{O}(f)$, where $f:\F_q^n\to\F_q$ is randomized}
\Ensure{$|\E[e(f(x)-\chi_\gamma(x))]| \ge \rho \Rightarrow$ \sr{true};  $\gamma$ and the coefficients of a target linear function are linearly independent $\Rightarrow$  \sr{false}} 
\State{$m := \lceil \frac{2q^4}{\rho^2}\ln \frac{q}{\delta} \rceil$}
\ForAll{$a\in\F_q$}
\State{$(x^{(1)},b^{(1)}),\ldots, (x^{(m)},b^{(m)}) \gets \mathbb{O}(f)$}
\State{\sr{if} $\sum_i\onen\{b^{(i)}-\chi_\gamma(x^{(i)})=a\} \ge (\frac{1}{q}+\frac{\rho}{2q^{2}})m$ \sr{then} \sr{return}(\sr{true})}\label{checkCor:condition}
\EndFor
\State{\sr{return}(\sr{false})}
\end{algorithmic}
\end{algorithm}

\begin{lemma}\label{lemma:checkCor}
Let $\chi_\alpha$ be a target linear function. The subroutine \sr{checkCor} outputs \sr{true} if the given $\gamma$ satisfies $|\E[e(f(x)-\chi_\gamma(x))]| \ge \rho$ with probability at least $1-\delta$. On the other hand, if $\gamma$ and $\alpha$ are linearly independent, \sr{checkCor} outputs \sr{false} with probability at least $1-\delta$ in time $poly(n,\rho^{-1},\ln{\delta^{-1}})$.
\end{lemma}

\begin{proof}
The lemma follows from Lemmas~\ref{lemma:how far the correlated pair} and~\ref{lemma:LI uniform} and the standard probabilistic argument. For the complete proof, see Appendix~\ref{proof: checkCor}.
\end{proof}

Algorithm~\ref{alg:LDME} is our main reduction from LDME to LBP. 
Let \sr{LBP}$(S,\rho)$ be a subroutine for solving LBP (of the degree $d$) with high probability. W.l.o.g., we can assume the failure probability is at most $1/4$ by constant number of repetitions.
\begin{algorithm}
\caption{Learning with Discrete Memoryless Errors (\sr{main2})}\label{alg:LDME}
\begin{algorithmic}[1]
\Require{$n,k\in\N$, $\rho\in(0,1)$, $\delta\in(0,1)$, $\mathbb{O}(f)$, where $f:\F_q^n\to\F_q$ is a randomized function}
\Ensure{$c\alpha\in\F_q^n$ for some $c\in\F_q\setminus\{0\}$, where $\chi_\alpha$ is a target linear function}
\ForAll{$\gamma \in \F_q^{n}$ of the size $|\gamma|=1$}
\State{\sr{if} \sr{checkCor}$(n,k,\gamma,\frac{\delta}{4n(q-1)},\mathbb{O}(f))$ \sr{then} \sr{return}($\gamma$)}\label{line:LDME: checkCor1}
\EndFor
\ForAll{consecutive partitions ($J$,$\bar{J}$) of $[n]$, $a_1,a_2\in\F_q$, $s_1,s_2\in\F_q\setminus\{0\}$}\label{line:LDME: select cp}
\State{\sr{repeat} $M=\lceil\log{2/\delta} \rceil$ times \sr{do}}
\State{\quad generate a light bulb instance $S$ as follows:}\label{line:LDME: make instance}
\Statex{\quad\quad(where the degree $d \ge \frac{8q^6}{\rho^2}\ln{4}$ is determined by the subroutine \sr{LBP})}
\Statex{\qquad\quad 1: (for $i$-th row,) get an example $(x,b)\gets\mathbb{O}(f)$}
\Statex{\qquad\quad 2: \sr{for all} $\alpha \in \F_q^J$, $\beta \in \F_q^{\bar{J}}$ where $1\le|\alpha|,|\beta|\le \lceil k/2 \rceil$, $init(\alpha) = s_1$, $init(\beta) = s_2$}
\Statex{\qquad\quad 3: \quad list all values $b - \chi_\alpha(x) -a_1$ and $\chi_\beta(x)$ (we regard $\alpha, \beta$ as indices of columns)}
\Statex{\qquad\quad 4: \sr{end for}}
\Statex{\qquad\quad 5: change entries taking $a_2$ to 1}
\Statex{\qquad\quad 6: change the other entries to $-1$ with probability $q/2(q-1)$, otherwise, $1$}
\State{\quad execute $(\gamma_1,\gamma_2) \gets$ \sr{LBP}$(S,\frac{\rho}{2q^3})$ ($\gamma_1,\gamma_2\in\F_q^n$ are indices of the correlated pair)}\label{line:LDME execution of LBP}
\State{\quad \sr{if} \sr{checkCor}$(n,k,\gamma_1+\gamma_2,\frac{\delta}{4Mnq^2(q-1)^2},\mathbb{O}(f))$ \sr{then} \sr{return}($\gamma_1+\gamma_2$)}\label{line:LDME: checkCor2}
\State{\sr{end} \sr{repeat}}
\EndFor
\end{algorithmic}
\end{algorithm}

The proof of Lemma~\ref{lemma:LDME} is informally given as mentioned in Section~\ref{section:overview LWEtoLBP}, and we give the complete proof in Appendix~\ref{proof:Main Lemma2}. Theorem~\ref{thorem:LWEtoLBP} immediately follows from Lemma~\ref{lemma:LDME} by substituting some constant for $\delta$.

\begin{lemma}\label{lemma:LDME}
Assume that the subroutine \sr{LBP} solves LBP for some $d \ge \Omega(\frac{\log{N}}{\rho^2})$ in time $T(N,\rho)$ w.h.p., where $N$ is the number of the vectors. Then the algorithm \sr{main2}$(n,k,\rho,\delta)$ solves LDME for any target linear function $\chi_{\alpha}$ $(1\le|\alpha|\le k)$ in time $\poly(n,\rho^{-1},\ln{\delta^{-1}})\cdot d\cdot T((qn)^{\frac{k}{2}},\frac{\rho}{2q^3})$ with probability at least $1-\delta$.
\end{lemma}
\section{Discussions and Future Directions}

We introduced the reduction from learning juntas over any finite fields to LBP, and gave the first non-trivial learning algorithm for such a class. Our results also enhance the motivation of designing an efficient algorithm for LBP, because it automatically improves the upper bound for learning $k$-juntas for not only the binary domain but also any finite field.

However, by our reduction, if we could construct a linear-time algorithm for LBP, the upper bound will be improved to at best $O(n^{\frac{k}{2}})$. Therefore, unlike in the binary case, it is open whether there exists a scenario that the polynomial factor can be improved to less than $k/2$. Remember that we first reduced the learning juntas problem to LDME which was the extension of the challenging learning problem, LWE. For further improvement, such a hard problem should be avoided.

In addition, our reduction makes extensive use of the properties of finite fields. Thus, it is also open whether we can design a non-trivial learning algorithm that works for any finite alphabet, in particular, $q=6$.



\bibliography{myrefs}

\appendix

\section{Proofs of Lemmas in Section~\ref{section:Prelimi}}

\subsection{Proof of Lemma~\ref{lemma:concecutive partitions}}\label{proof:concecutive partitions}
For convenience, we say $i\in[n]$ is supportive if $\alpha_i\neq 0$. For $i\in[n]$, let $J_i\subset [n]$ be a subset which consists of cyclically consecutive $\lceil n/2\rceil$ coordinates from $i$, and $m_i$ be the number of supportive coordinates contained in $J_i$. For $J_1$, the remaining $\lfloor n/2 \rfloor$ coordinates contain $k-m_1$ supportive coordinates, thus $k-m_1\le m_{\lceil n/2\rceil+1}\le k-m_1+1$ (because $J_{\lceil n/2\rceil+1}$ also contains the first coordinate in the case where $n$ is odd). If $m_1 = \lceil k/2 \rceil$, then $(J_1,\bar{J_1})$ is a desired partition. So we assume that $m_1 \neq \lceil k/2 \rceil$. If $m_1 \le \lceil k/2\rceil -1$, we have $m_{\lceil n/2\rceil+1} \ge k-m_1\ge \lfloor k/2 \rfloor +1  \ge \lceil k/2\rceil$. Otherwise if $m_1 \ge \lceil k/2\rceil+1$, we have $m_{\lceil n/2\rceil+1}\le k-m_1+1 \le \lfloor k/2 \rfloor$. Since the difference between $m_i$ and $m_{i+1}$ must be $0$ or $\pm 1$, there exist at least one coordinate $i$ satisfying $m_i = \lceil k/2 \rceil$ in any cases.

\subsection{Proof of Lemma~\ref{lemma:A-projection}}\label{proof:A-projection}
Let $g:\F_q^n\to\C$ be the right-hand side of~(\ref{eq:A pro}). It is enough to show that for any $\alpha\in \F_q^n$,
\[\fc{g}(\alpha) = \fc{f^a_A}(\alpha).\]
From the definition of $\fc{g}(\alpha)$, it follows that
\begin{align*}
\fc{g}(\alpha) = \E_x[g(x)\ol{e(\chi_\alpha(x))}] &= \E_x[\E_z[e(af(x+A^Tz))\ol{e(\chi_{a^m}(z))}]\ol{e(\chi_\alpha(x))}]\\
&= \E_z[\E_x[e(af(x+A^Tz))\ol{e(\chi_\alpha(x+A^Tz))}]e(\chi_{\alpha}(A^Tz))\ol{e(\chi_{a^m}(z))}] \\
&= \fc{af}(\alpha) \E_z[e(\chi_{A\alpha}(z))\ol{e(\chi_{a^m}(z))}]\\
&= \fc{af}(\alpha)\onen\{A\alpha = a^m\} = \fc{f^a_A}(\alpha).
\end{align*}
For the second part, notice that for any $x\in\F_q^n$ and $z\in\F_q^m$, exactly one element $y_z \in \F_q^n$ satisfying $y_z-A^Tz=x$ is determined. Therefore,
\begin{align*}
\E_{b_x}[e(ab_x)] &= \sum_{z\in\F_q^m} q^{-m}\left(e(af(y_z)-a\sum z_i)\right)\\
&= \E_{z}[e(af(x+A^Tz)-\chi_{a^m}(z))]\\
&= \E_{z}[e(af(x+A^Tz))\ol{e(\chi_{a^m}(z))}] = f^a_A(x) \quad (\because~(\ref{eq:A pro})).
\end{align*}

\section{Proofs of Lemmas in Section~\ref{section:Junta to LWE}}\label{proofs in section juntas to lwe}

\subsection{Proof of Lemma~\ref{lemma:const}}\label{proofs for subroutines:const}
If $f$ is constant, then the algorithm obviously outputs the value with probability 1. If $f$ is not constant, then there are two entries which have different values in the truth table of $f$. The probability that each value appears is at least $q^{-k}$ because the value of the truth table is affected by only at most $k$ coordinates. If $m$ examples contain these values as their labels, then the algorithm will output $\bot$. The probability that each value does not appear in $m$ labels is bounded above by $(1-q^{-k})^m \le e^{-m/q^{k}} \le \frac{\delta}{2}$. By the union bound, the failure probability is at most $\delta$.

\subsection{Proof of Lemma~\ref{lemma:checkRC}}\label{proofs for subroutines:checkRC}
First, we consider the case where $\fc{f}(\alpha) \neq 0$. Assume that $\Pr[\Tr(f(x)-\chi_\alpha(x))=a] < \frac{1}{p} + \frac{1}{q^k}$ for all $a\in\F_p$. Since $\alpha$ does not have nonzero value at irrelevant coordinates by Fact~\ref{fact:fcnonzero to relevant}, the value $f-\chi_\alpha$ is determined by at most $k$ coordinates of $x$, and $\Pr[\Tr(f(x)-\chi_\alpha(x))=a] \le \frac{1}{p}$ for all $a\in\F_p$. This implies $\Pr[\Tr(f(x)-\chi_\alpha(x))=a] = \frac{1}{p}$ for all $a\in\F_p$ and $\fc{f}(\alpha) = 0$, which is contradiction. Thus, there exists $a'\in\F_p$ such that $\Pr[\Tr(f(x)-\chi_\alpha(x))=a'] \ge \frac{1}{p} + \frac{1}{q^k}$. By the Hoeffding inequality, the probability that the condition in line~\ref{line:checkRC condtion} does not hold w.r.t. $a'$ is bounded above by $e^{-\frac{m}{2q^{2k}}} \le \frac{\delta}{p} < \delta$.

On the other hand, if there exists $i\in[n]$ such that $i$ is irrelevant and $\alpha_i\neq 0$, then for any $a_q\in\F_q$,
\[\Pr[f(x)-\chi_\alpha(x)=a_q] = \sum_{v\in\F_q} \Pr[f(x)-\chi_{\alpha'}(x) = v]\Pr[\alpha_ix_i=a_q-v] = \frac{1}{q},\]
where $\alpha'_i = 0$ and $\alpha'_j = \alpha_j$ for $j\neq i$. For any $a_p\in\F_p$, this implies 
\[\Pr[\Tr(f(x)-\chi_\alpha(x))=a_p] = \sum_{a_q\in\Tr^{-1}(a_p)} \Pr[f(x)-\chi_{\alpha'}(x) = a_q] = \frac{|\Tr^{-1}(a_p)|}{q}=\frac{p^{\ell-1}}{p^{\ell}} = \frac{1}{p}.\]
By the Hoeffding inequality, the probability that the condition in line~\ref{line:checkRC condtion} holds is bounded above by $e^{-\frac{m}{2q^{2k}}} \le \frac{\delta}{p}$. Therefore, by the union bound, the probability that the condition holds for some $a_p\in\F_p$ (i.e., the failure probability) is at most $\delta$. 

\subsection{Proof of Lemma~\ref{lemma:addRC}}\label{section:proof of lemma addRC}

In this section, we show the correctness of the subroutine \sr{addRC}. First, we introduce the following simple fact. The reader may skip the proof of the Claim~\ref{claim:two parities} because it is quite basic and not essential.

\begin{claim}\label{claim:two parities}
For any vectors $\alpha,\beta\in\F_q^n\setminus\{0^n\}$, the following holds:\\

(i) If $\beta \neq c\alpha$ for any $c\in\F_q$ (i.e., $\alpha$ and $\beta$ are linearly independent), then for any $a,b\in\F_q$,
\[\Pr_x[x^T\alpha=a \text{ and } x^T\beta=b]= \frac{1}{q^2}.\]

(ii) If $\beta = c\alpha$ ($c\neq 0$), then for any $a,b\in\F_q$,
\[\Pr_x[x^T\alpha=a \text{ and } x^T\beta=b]= 
\begin{cases}
1/q &(\text{if } b=ca)\\
0 &(\text{otherwise}).
\end{cases}\]
\end{claim}

In other words, if $\alpha,\beta (\neq 0^n)$ satisfies the condition~(i), then $\chi_{\alpha}(x)$ and $\chi_{\beta}(x)$ are uniformly and pairwise independently distributed w.r.t. the uniform selection of $x\in\F_q^n$.

\begin{proof}
(i) If $\beta \neq c\alpha$ for any $c\in\F_q$, there are two coordinates $i,j\in[n]$ satisfying $\beta_i = c\alpha_i$, $\beta_j = c'\alpha_j$, $c\neq c'$, and $\alpha_i,\alpha_j \neq 0$. First we select values in $\F_q^{[n]\setminus\{i,j\}}$, and for any choice, the remaining condition takes the following form: for some $ v_1,v_2 \in \F_q$,
\[\alpha_ix_i + \alpha_jx_j = v_1 \text{ and } c\alpha_ix_i + c'\alpha_jx_j = v_2.\]
Since $\alpha_ic'\alpha_j - \alpha_jc\alpha_i = \alpha_i\alpha_j(c'-c) \neq 0$, the above equations have a unique solution w.r.t. $(x_i,x_j)$. The probability that they take the values of the unique solution is exactly $q^{-2}$.
\\
\\
(ii) If $\beta = c\alpha$ ($c\neq 0$), the condition takes the following form:
\begin{align*}
\alpha_1x_1 +\cdots + \alpha_n x_n &= a \\
\alpha_1x_1 +\cdots + \alpha_n x_n &= c^{-1}b
\end{align*}
Obviously, the probability is $q^{-1}$ if $a=c^{-1}b$, otherwise, the probability is $0$.
\end{proof}

Next, we show that for small subspace $\F_q^D$, only one vector $\alpha\in\F_q^D$ satisfies $A\alpha=1^m$ with non-negligible probability w.r.t. the uniform selection of $A$.

\begin{claim}\label{claim:A projected size is small}
For any subset $D\subseteq [n]$ $(|D|\le k)$, $\alpha \in \F_q^D\setminus\{0^{n}\}$, and $m \ge k$,
\[\Pr_{A \sim \F_q^{m\times n}}[A\alpha = 1^m \text{ and } \, A\beta \neq 1^m \text{ for each } \beta\in\F_q^D\setminus\{\alpha\}] \ge \frac{q^{m-k}-1}{q^{2m-k}}\]
Especially, if the parameter $m$ is selected as $m=k+1$, then
\[\Pr_{A \sim \F_q^{(k+1)\times k}}[A\alpha = 1^{k+1} \text{ and } \,A\beta \neq 1^{k+1} \text{ for each } \beta\in\F_q^D\setminus\{\alpha\}] \ge \frac{1}{q^{k+2}}\]
\end{claim}

\begin{proof}
The second part immediately follows from the first one, thus we give only a proof of the first part. It is sufficient to show that
\begin{equation}\label{eq1}
\Pr_A[A\alpha = 1^m] = \frac{1}{q^m} \text{ and } \Pr_A[A\beta \neq 1^m \text{ for each } \beta\in\F_q^D\setminus\{\alpha\}|A\alpha = 1^m] \ge 1-\frac{1}{q^{m-k}}.
\end{equation}
Since $\alpha\neq 0^n$, $\Pr_{x\sim\F_q^n}[x^T\alpha = 1] = q^{-1}$ holds, thus we have $\Pr_A[A\alpha = 1^m] = q^{-m}$. By Claim~\ref{claim:two parities}, for any $\beta\neq \alpha$, we have
\[\Pr_x[x^T\beta = 1 \text{ and } x^T\alpha = 1] \le \frac{1}{q^2}.\]
Therefore,
\[\Pr_x[x^T\beta = 1|x^T\alpha = 1] = \frac{\Pr_x[x^T\beta = 1 \text{ and } x^T\alpha = 1]}{\Pr_x[x^T\alpha = 1]} \le \frac{q}{q^2} = \frac{1}{q}\]
and
\[\Pr_A[A\beta = 1^m|A\alpha = 1^m] \le \frac{1}{q^m}.\]
Since $|D|\le k$, the number of vectors $\beta\in\F_q^D$ is at most $q^k$. Hence, by the union bound,
\[\Pr_A[\exists \beta \in \F_q^{D}\setminus\{\alpha\} \text{ s.t. }  A\beta = 1^m|A\alpha = 1^m] \le \frac{q^k}{q^m},\]
which is equivalent to the second part of the inequality~(\ref{eq1}).
\end{proof}

Let $f:\F_q^n\to \F_q$ be $k$-junta and $D\subseteq [n]$ be the set of relevant coordinates of $f$.
In the following claims, we assume that there exists $\alpha\in\F_q^D\setminus\{0^{n}\}$ satisfying $\fc{f}(\alpha)\neq 0$ and the event in Claim~\ref{claim:A projected size is small} occurs for $D$, $\alpha$, and $m=k+1$. By the definition of $(A,a)$-projection, the projected function satisfies $f^a_A \equiv \fc{af}(\alpha)e(a\chi_\alpha)$ for any $a \in\F_q\setminus\{0\}$, because $af$ has the same domain $D$. In addition, we assume that the example $(x,b)$ is simulated as follows: for $y\gets_u \F_q^n \text{ and } z\gets_u \F_q^{k+1}$,
\[(x,b) := \left(y-A^Tz, f(y)-\sum_j z_j\right).\]

\begin{claim}\label{claim:noisy parity learning}
Let $\alpha\in \F_q^n$. If the $(a,A)$-projected function satisfies $f^{a}_A \equiv \fc{af}(a\alpha)e(a\chi_\alpha)$ for all $a\in\F_q\setminus\{0\}$, and the example $(x,b)$ is simulated as the above, then the conditional distribution of $b_x$ is determined by only the value of $\chi_{\alpha}(x)$, that is, for $x,x'\in\F_q^n$, if $\chi_{\alpha}(x) = \chi_{\alpha}(x')$, then $SD(b_x,b_{x'}) = 0$.
\end{claim}

\begin{proof}By Lemma~\ref{lemma:A-projection} and the assumption, $\E[e(ab_x)] = f^{a}_A(x) = \fc{af}(a\alpha)e(a\chi_{\alpha}(x))$ for any $a \in\F_q\setminus\{0\}$. By Fact~\ref{fact:sdcd},
\begin{align*}
&\chi_{\alpha}(x) = \chi_{\alpha}(x') \\&\quad\Longrightarrow \E[e(ab_x)] = \fc{af}(a\alpha)e(a\chi_{\alpha}(x)) =\fc{af}(a\alpha)e(a\chi_{\alpha}(x')) = \E[e(ab_{x'})] \text{ for any } a\in\F_q\\
&\quad\Longleftrightarrow CD(b_x,b_{x'}) = \max_{a\in\F_q} \left|\E[e(ab_x)]-\E[e(ab_{x'})]\right| = 0\Longleftrightarrow SD(b_x,b_{x'}) = 0.
\end{align*}
\end{proof}

In the algorithm \sr{addRC}, an example of LDWE is simulated as $(x,a\cdot b)$ for some $a\in\F_q\setminus\{0\}$. Obviously if the distribution of $b_x$ is determined, then the distribution of $a\cdot b_x$ is also determined. In addition, it is also obvious that the value of $\chi_{\alpha}(x) = a^{-1}\chi_{a\alpha}(x)$ is determined by the value of $\chi_{a\alpha}(x)$. 
Therefore, the above claim implies that the simulated oracle in the algorithm \sr{addRC} returns indeed an instance of LDME for the target linear function $\chi_{a\alpha}(x)$. Finally, we show that the simulated instance has a large correlation with the linear function $\chi_{a\alpha}$ if the algorithm \sr{addRC} chooses a ``good'' $a\in\F_q\setminus\{0\}$.

\begin{claim}\label{claim:correlation is large}
We assume the same notations and conditions as in Claim~\ref{claim:noisy parity learning}. In addition, if the $k$-junta function $f$ satisfies $\fc{f}(\alpha)\neq 0$ and the parameter $m$ is selected by $m=k+1$, (i.e., $A\in\F_q^{(k+1)\times n}$), then
\[\max_{a\in\F_q\setminus\{0\}}\mathrm{Cor}(ab,\chi_{a\alpha})  \ge \frac{1}{q^{k+1}}.\]
\end{claim}

\begin{proof}
For any $a\in\F_q\setminus\{0\}$,
\begin{align*}
\mathrm{Cor}(ab,\chi_{a\alpha}) &:= |\E_{x,b}[e(ab)\ol{e(\chi_{a\alpha}(x))}]|\\
&= |\E_x[\E_b[e(ab_x)]\ol{e(\chi_{a\alpha}(x))}]|\\
&= |\E_x[f^a_A(x)\ol{e(\chi_{a\alpha}(x))}]| \qquad(\because \text{Lemma~\ref{lemma:A-projection}})\\
&= |\fc{af(a\alpha)}\E_x[e(\chi_{a\alpha}(x))\ol{e(\chi_{a\alpha}(x))}]| = |\fc{af}(a\alpha)| \qquad(\because \text{the assumption in Claim~\ref{claim:noisy parity learning}})\\
\end{align*}
Thus, it is enough to show that $\max_{a\in\F_q\setminus\{0\}} |\fc{af}(a\alpha)| \ge 1/q^{k+1}$. Let $U^{(1)}_q,\ldots,U^{(n)}_q$ and $U'_q$ be independently and uniformly distributed random variables over $\F_q$, and let $U^n_q = (U^{(1)}_q,\ldots,U^{(n)}_q)$.
\begin{align*}
\max_{a\in\F_q\setminus\{0\}} |\fc{af}(a\alpha)| &= \max_{a\in\F_q\setminus\{0\}} \left|\E[e(a(f(U^n_q)-\chi_{\alpha}(U^n_q)))]\right|\\
&= \max_{a\in\F_q\setminus\{0\}} \left|\E[e(a(f(U^n_q)-\chi_{\alpha}(U^n_q)))]-\E[e(aU'_q)]\right| \qquad(\because \E[e(aU'_q)] = 0)\\
&= \max_{a\in\F_q} \left|\E[e(a(f(U^n_q)-\chi_{\alpha}(U^n_q)))]-\E[e(aU'_q)]\right|\\
&= CD(f(U^n_q)-\chi_{\alpha}(U^n_q),U'_q) \\
&\ge \frac{1}{\sqrt{q-1}}\cdot 2\cdot SD(f(U^n_q)-\chi_{\alpha}(U^n_q),U'_q) \qquad(\because \text{Fact~\ref{fact:sdcd}})
\end{align*}
By the assumption, $\E[e(f(U^n_q)-\chi_{\alpha}(U^n_q))] = \fc{f}(\alpha) \neq 0$. Since $\E[e(U'_q)] = 0$ , they must not be statistically identical, that is, $SD(f(U^n_q)-\chi_{\alpha}(U^n_q),U'_q)\neq 0$. In addition, by Fact~\ref{fact:fcnonzero to relevant}, $f(x)-\chi_{\alpha}(x)$ is $k$-junta. Therefore, by the definition of statistical distance, $2\cdot SD(f(U^n_q)-\chi_{\alpha}(U^n_q),U'_q) \ge 1/q^k$. Now we have
\[\max_{a\in\F_q\setminus\{0\}} |\fc{af}(a\alpha)| \ge \frac{1}{\sqrt{q-1}}\cdot 2\cdot SD(f(U^n_q)-\chi_{\alpha}(U^n_q),U'_q) \ge \frac{1}{\sqrt{q-1}}\cdot\frac{1}{q^k} \ge \frac{1}{q^{k+1}}.\]
\end{proof}

Now we give the proof of Lemma~\ref{lemma:addRC}.

\begin{proof}[Proof (Lemma~\ref{lemma:addRC})] 
First, for simplicity, let us assume that execution of \sr{checkRC} succeeds with probability 1. 
If $f$ is not constant and some relevant coordinates are not contained in $R$, then there exists a restriction $\tau$ on $R$ such that the restricted function $f|_{\tau}$ is not constant. In this case, by Lemma~\ref{lemma:simuljunta}, there exist $j\in[p^{\ell-1}]$ and $\alpha \in \F_q^{|\bar{R}|}$ such that $|\alpha| \ge 1$ and $\fc{c_jf|_{\tau}}(\alpha) \neq 0$. 

For convenience, we regard $f$ as the restricted function as $f := c_jf|_{\tau}$.
For the set $D$ of relevant coordinates of $f$, $|D|\le k$. By Claim~\ref{claim:A projected size is small} and the argument following Claim~\ref{claim:A projected size is small}, for all $a' \in \F_q\setminus\{0\}$, $f^{a'}_A \equiv \fc{a'f}(a'\alpha)e(a'\chi_\alpha)$ with probability at least $1/q^{k+2}$ w.r.t. the uniform selection of $A$. Since \sr{addRC} tries to select $A$ more than $q^{k+2}\ln{4/\delta}$ times, at least one of selected $A$'s satisfies this condition with probability at least $1-\delta/4$. Thus in the following argument, we assume that the algorithm \sr{addRC} succeeds in selecting such an $A$.

If the algorithm \sr{addRC} succeeds in selecting the above matrix $A$, then by Claims~\ref{claim:noisy parity learning} and~\ref{claim:correlation is large}, there exists $a\in\F_q\setminus\{0\}$ such that the simulated noisy example in line~\ref{line:addRC Aa projection} corresponds to the example from LDME of the correlation $\rho\ge 1/q^{k+1}$. By the assumption, the repetition of \sr{LDME} recovers $\alpha$ up to constant factor (i.e., finds $a'\alpha$ for some $a'\in\F_q$) with probability at least $1-\delta/4$. If LDME is solved successfully, then at least one relevant coordinate is added to $R$ in line~\ref{line:addRC add RV}.

If the algorithm \sr{addRC} fails in selecting $A$ and $a$, the subroutine \sr{LDME} may return some undesirable candidate. In this case, the subroutine \sr{checkRC} returns \sr{false} in line~\ref{line:addRC:check candidate}, and irrelevant coordinates are not added to $R$. Therefore, by the union bound, the failure probability is at most $\delta/4 + \delta/4 = \delta/2$ under the condition that \sr{checkRC} succeeds with probability 1.

In fact, our algorithm \sr{checkRC} may fail. Since the number of executions of \sr{checkRC} is at most $Mq^{k+3}$, by the union bound, the probability that some executions of \sr{checkRC} fail is at most $\delta/2$. Thus, the total failure probability is at most $\delta$. The total running time is bounded above by 
\[M\cdot \poly(n,q^k,\ln{1/\delta}) \cdot T(n,k,1/q^{k+1})  = \poly(n,q^k,\ln{1/\delta}) \cdot T(n,k,1/q^{k+1}).\]
\end{proof}

\section{Proofs of Lemmas in Section~\ref{section:LWE to LBP}}

\subsection{Proof of Lemma~\ref{lemma:how far the correlated pair}}\label{proofs of claims:how far the cor}

For simplicity, let $p_a := \Pr[X=a]$ for $a\in\F_q$.
First we show that
\[|\E[e(X)]|\ge \rho \Longrightarrow \exists a\in \F_q \>s.t.\> \left|p_a-\frac{1}{q}\right|\ge\frac{\rho}{q}.\]
By contraposition, we assume that $|p_a-\frac{1}{q}| <\frac{\rho}{q}$ for any $a\in\F_q$. Then,
\begin{align*}
|\E[e(X)]| = \left|\sum_{a\in\F_q}p_ae(a)\right| &= \left|\sum_{a\in\F_q} (p_a -\frac{1}{q}) e(a)\right| \\
&\le \sum_{a\in\F_q} \left|p_a -\frac{1}{q}\right| |e(a)| \\
&< \frac{\rho}{q} \cdot \sum_{a\in\F_q} |e(a)| = \rho,
\end{align*}
where the second equality follows from the fact that $\sum_{a\in\F_q}{e(a)} = 0$.

Now we have that $|p_a-\frac{1}{q}|\ge\frac{\rho}{q}$ for some $a\in\F_q$. If $p_a-\frac{1}{q} \ge\frac{\rho}{q}$, then $p_a \ge \frac{1}{q} +\frac{\rho}{q} \ge \frac{1}{q} + \frac{\rho}{q^2}$. Therefore, the remaining case is that $p_a \le \frac{1}{q}  -\frac{\rho}{q}$. In this case,
\[(q-1) \max_{b\in \F_q\setminus\{ a\}} p_b \ge \sum_{b\neq a}p_b = 1-p_a \ge \frac{q-1}{q} +\frac{\rho}{q}.\]
Thus, there exists $b\in\F_q$ such that $p_b\ge \frac{1}{q} +\frac{\rho}{q(q-1)} \ge \frac{1}{q} +\frac{\rho}{q^2}.$ 

\subsection{Proof of Lemma~\ref{lemma:LI uniform}}\label{proofs of claims:li uniform}
The lemma immediate follows from Claim~\ref{claim:two parities} as follows:
\begin{align*}
\Pr[X-\chi_\beta(x) = a] &= \sum_{v\in\F_q}\sum_{v'\in\F_q}\Pr[\chi_\alpha(x) = v,\chi_{\beta}(x)=v']\Pr[X=a+v'|\chi_{\alpha}(x)=v]\\
&= \frac{1}{q^2} \sum_{v\in\F_q}\sum_{v'\in\F_q}\Pr[X=a+v'|\chi_{\alpha}(x)=v] \quad(\because Claim~\ref{claim:two parities})\\
&= \frac{1}{q^2} \sum_{v\in\F_q} 1 = \frac{1}{q}.
\end{align*}

\subsection{Proof of Lemma~\ref{lemma:checkCor}}\label{proof: checkCor}
If $\E[e(f(x)-\chi_{\gamma}(x))]\ge\rho$, then by Lemma~\ref{lemma:how far the correlated pair}, there exists $a\in\F_q$ such that $\Pr[f(x)-\chi_{\gamma}(x) = a] \ge 1/q + \rho/q^2$. Since \sr{checkCor} tries all $a\in\F_q$, by Hoeffding inequality, the condition in line~\ref{checkCor:condition} is not satisfied with probability at most
\[\exp\left(-\frac{2\rho^2}{4q^4}m\right) \le \exp\left(-\frac{\rho^2}{2q^4}\cdot \frac{2q^4}{\rho^2} \ln\frac{q}{\delta}\right) = \frac{\delta}{q} \le \delta.\]
On the other hand, if $\chi_{\alpha}$ is a target linear function, and $\gamma$ and $\alpha$ are linearly independent, then by Lemma~\ref{lemma:LI uniform}, $\Pr[f(x)-\chi_{\gamma}(x) = a] = 1/q$ for each $a\in\F_q$. By Hoeffding inequality and the union bound, the error probability that the condition in line~\ref{checkCor:condition} is satisfied is at most
\[q \cdot \exp\left(-\frac{2\rho^2}{4q^4}m\right) \le  q\cdot \frac{\delta}{q} = \delta.\]

\subsection{Proof of Lemma~\ref{lemma:LDME}}\label{proof:Main Lemma2}
In this section, we show the correctness of the algorithm \sr{main2}. We use $\alpha$ to denote the coefficients of the target linear function, that is, the distribution of the target randomized function $f(x)$ is determined only by $\chi_{\alpha}(x)$ for each $x\in\F_q^n$. We assume that a partition $(J,\bar{J})$ is consecutive and divides a nonzero part of $\alpha$ into half as in Lemma~\ref{lemma:concecutive partitions}.

We begin with the analysis of non-target pairs for each row in the reduced instance.

\begin{claim}\label{claim:case1}
If a partition $(J,\bar{J})$ and linearly independent vectors $\beta, \beta'\in \F_q^J\cup\F_q^{\bar{J}}$ satisfy that 
$\alpha^J \neq 0^n$, $\alpha^{\bar{J}} \neq 0^n$, and for any $a\in\F_q\setminus\{1\}, \alpha^{J} \neq a\beta, \alpha^{J} \neq a\beta', \alpha^{\bar{J}}\neq a\beta, \alpha^{\bar{J}}\neq a\beta'  \text{ and } \beta+\beta' \neq \alpha$, 
then $\chi_{\beta}$ and $\chi_{\beta'}$ are uniformly and pairwise independently distributed under any condition about $\chi_{\alpha}$, i.e., for any $ v_1,v_2,v_3 \in \F_q$,
\[\Pr_{x}[\chi_{\beta}(x)= v_1 \text{ and } \chi_{\beta'}(x) = v_2|\chi_{\alpha}(x)=v_3] = \frac{1}{q^2}.\]
\end{claim}

The proof of Claim~\ref{claim:case1} is not so essential, thus the reader may skip over it.

\begin{proof}
Since $\alpha \neq 0^n$, it is enough to show that, for any $v_1,v_2,v_3 \in \F_q$,
\[\Pr_{x}[\chi_{\beta}(x)= v_1, \chi_{\beta'}(x) = v_2, \chi_{\alpha}(x)=v_3] = \frac{1}{q^3}.\]
W.l.o.g., we can assume that $\beta\in\F_q^J$ and $\beta \neq \alpha^{J}$ (in this case, either $\beta' = \alpha^{J}$ or $\beta' = \alpha^{\bar{J}}$ may hold). First consider the case where $\beta'\in\F_q^{\bar{J}}$. We select three coordinates $(i_1,i_2,i_3)$ as follows: by linearly independence of $\beta$ and $\alpha^J$, we can select $(i_1,i_2)$ such that $(\alpha_{i_1},\alpha_{i_2})$ and $(\beta_{i_1},\beta_{i_2})$ are also linearly independent. Then, we select $i_3\in\bar{J}$ to satisfy that $\beta'_{i_3} \neq 0$. Now we have the three vectors $\{(\alpha_{i_1},\alpha_{i_2},\alpha_{i_3}), (\beta_{i_1},\beta_{i_2}, 0), (0, 0,\beta'_{i_3})\}$. It is not so difficult to see that they  are linearly independent. 

Otherwise if $\beta'\in\F_q^{J}$, we select $i_3$ satisfying $\alpha_{i_3} \neq 0$, and we can select $(i_1,i_2)$ such that $(\beta_{i_1},\beta_{i_2})$ and $(\beta'_{i_1},\beta'_{i_2})$ are also linearly independent.  Then we have three vectors $\{(\alpha_{i_1},\alpha_{i_2},\alpha_{i_3}),$ $(\beta_{i_1},\beta_{i_2}, 0), (\beta'_{i_1},\beta'_{i_2},0)\}$ which are also linearly independent. 

In any case, for any assignment to $[n]\setminus\{i_1,i_2,i_3\}$, the solution of the remaining linear system in $x_{i_1},x_{i_2},x_{i_3}$ is uniquely determined, and the claim holds as in the proof of Claim~\ref{claim:two parities}. 
\end{proof}

In the reduction, we assume that the initial values $s_1$ and $s_2$ are consistent with $\alpha$, that is, $s_1= init(\alpha^J)$ and $s_2= init(\alpha^{\bar{J}})$. Any pair of indices $(\beta,\beta')$ except for $(\alpha^{J},\alpha^{\bar{J}})$  satisfies the conditions in Claim~\ref{claim:case1}, because they are non-zero and their initial values are fixed. In addition, the value of $f(x)$ depends on only the value of $\chi_{\alpha}$. Therefore, by Claim~\ref{claim:case1}, the pair of entries indexed by $(\beta,\beta')$ are also uniformly and independently distributed.

For an element $a\in\F_q$ and a random variable $X$ taking values in $\F_q$, we use $\qtob{X}{a}$ to denote a $\{\pm1\}$-valued random variable given by operation in line~\ref{line:LDME: make instance} of \sr{main2}, i.e.,
\begin{enumerate}
	\item[(1)] if $X$ takes $a$, set as $\qtob{X}{a} = 1$,
	\item[(2)] otherwise, flip a biased coin with the head probability $p_h=q/(2(q-1))$, and if it comes up with head (resp. tail), set as $\qtob{X}{a}=-1$, (resp. $\qtob{X}{a}=1$).
\end{enumerate}
For any $a\in\F_q$, if $X$ is uniformly distributed over $\F_q$, then $\Pr[\qtob{X}{a}=1] = \frac{q-1}{q}\cdot \frac{q}{2(q-1)} = \frac{1}{2}$. Moreover, it is easy to see that if $X,Y$ are uniformly and pairwise independently distributed, then $\qtob{X}{a}$ and $\qtob{Y}{a}$ are also uniformly and pairwise independently distributed over $\{\pm1\}$. Therefore, any pair of entries indexed by $(\beta,\beta') \neq (\alpha^{J},\alpha^{\bar{J}})$ is selected uniformly and independently.\\

Now we move on to the analysis of the target pair, that is, the pair of entries corresponding to $(\alpha^{J},\alpha^{\bar{J}})$.

\begin{claim}\label{claim:correlated}
Let $(J,\bar{J})$ be any  partition  of $[n]$. If a randomized function $f:\F_q^n\to\F_q$ has a correlation with $\chi_{\alpha}$ as $\mathrm{Cor}(f,\chi_{\alpha})\ge \rho$, then there exist $a_1,a_2 \in \F_q$ such that
\[\Pr_{x,f}[f(x)-\chi_{\alpha^J}(x)-a_1 = a_2\text{ and } \chi_{\alpha^{\bar{J}}}(x) = a_2] \ge \frac{1}{q^2} + \frac{\rho}{q^{3}}.\]
\end{claim}
\begin{proof}
By Lemma~\ref{lemma:how far the correlated pair}, $\mathrm{Cor}(f,\chi_{\alpha})\ge \rho$ implies that there exists $a_1\in\F_q$ such that
\[\Pr_{x,f}[f(x)-\chi_{\alpha}(x) = a_1] \ge \frac{1}{q} + \frac{\rho}{q^{2}}.\]
Therefore,
\begin{align*}
\frac{1}{q} + \frac{\rho}{q^{2}} &\le \Pr_{x,f}[f(x)-\chi_{\alpha}(x) = a_1] \\
&= \Pr_{x,f}[f(x)-\chi_{\alpha^{J}}(x) - a_1 = \chi_{\alpha^{\bar{J}}}(x)] \\
&\le q \cdot \max_{a_2\in\F_q} \Pr_{x,f}[f(x)-\chi_{\alpha^{J}}(x) - a_1 = \chi_{\alpha^{\bar{J}}}(x) = a_2]
\end{align*}
\end{proof}

Then we estimate the correlation between the target pair in the reduced instance.

\begin{claim}\label{claim:correlation remain in the target}
Let $a\in\F_q$ and $\mu\in[0,1]$. If random variables $X$ and $Y$ in $\F_q$ satisfies
\[\Pr[X=a, Y=a]\ge \frac{1}{q^2} +\mu \text{ and } \Pr[X=a] = \Pr[Y=a] = \frac{1}{q},\]
then,
\[\Pr[\qtob{X}{a}\cdot\qtob{Y}{a} = 1] \ge \frac{1}{2} + 2p_{h}^{2}\mu,\]
where $p_{h} = \frac{q}{2(q-1)}$ as in the definition of $\qtob{X}{a}$.
\end{claim}

\begin{proof}
Let $p_1,p_2,p_3,p_4$ denote probabilities as
\[p_1= \Pr[X=a, Y=a], \quad p_2 = \Pr[X=a, Y\neq a],\]
\[p_3 = \Pr[X\neq a, Y = a],\quad p_4 = \Pr[X\neq a, Y\neq a].\]
Then, it follows that $p_1+p_2+p_3+p_4 = 1$, $p_1\ge\frac{1}{q^2} +\mu$, and 
\[p_4 = 1-\Pr[X=a] -\Pr[Y=a] + \Pr[X=a,Y=a] \ge 1-\frac{2}{q} +\frac{1}{q^2} +\mu = \left(1-\frac{1}{q}\right)^2 + \mu.\]
Therefore, the probability is bounded below by
\begin{align*}
\Pr[\qtob{X}{a}\cdot\qtob{Y}{a} = 1] &= \Pr[\qtob{X}{a} = \qtob{Y}{a}]\\ &= p_1\cdot 1 + (p_2+p_3)\cdot (1-p_h) + p_4\cdot(p_h^2+(1-p_h)^2)\\
&=(1-p_h) + p_1\cdot p_h + p_4\cdot(2p_h^2-p_h)\\
&\ge(1-p_h) + \frac{1}{q^2} p_h + \left(1-\frac{1}{q}\right)^2(2p_h^2-p_h) + \mu\cdot (p_h + 2p_h^2-p_h)\\
&= \frac{1}{2} + 2p_h^2\mu.
\end{align*}
\end{proof}
For our settings, take $X=f(x)-\chi_{\alpha^J}(x)-a_1$, $Y = \chi_{\alpha^{\bar{J}}}(x)$, and $\mu = \rho/q^3$. Then we have
\[\Pr[\qtob{X}{a}\cdot\qtob{Y}{a} = 1] \ge \frac{1}{2} + 2\frac{q^2}{4(q-1)^2}\frac{\rho}{q^3} \ge \frac{1}{2} + \frac{\rho}{2q^3},\]
and
\[\E[\qtob{X}{a}\cdot\qtob{Y}{a}] = 2\Pr[\qtob{X}{a}\cdot\qtob{Y}{a} = 1] -1 \ge \frac{\rho}{q^3}.\]
Therefore, if we take sufficiently many samples, then the target pair has a correlation at least $\frac{\rho}{2q^3}$ w.h.p.
Now we give the proof of Lemma~\ref{lemma:LDME}.

\begin{proof}[Proof (Lemma~\ref{lemma:LDME})]
As in the proof of Lemma~\ref{lemma:addRC}, we assume that all executions of \sr{checkCor} will succeed. Under the condition, even if an incorrect candidate is found in brute-force search in $(J,\bar{J}),a_1,a_2,s_1$, and $s_2$, the algorithm \sr{main2} does not output such an incorrect answer by Lemma~\ref{lemma:checkCor}. In fact, it is easily checked that the number of executions of \sr{checkCor} in lines~\ref{line:LDME: checkCor1} and~\ref{line:LDME: checkCor2} is at most $n(q-1)$ and $nq^2(q-1)^2\cdot M$, respectively. Therefore, by the union bound, the probability that at least one execution fails is bounded above by \[n(q-1) \cdot \frac{\delta}{4n(q-1)} + nq^2(q-1)^2M \cdot \frac{\delta}{4Mnq^2(q-1)^2} \le \frac{\delta}{2}.\]

Let $\alpha\in\F_q^n$ be the coefficients of the target linear function and $f$ be the target randomized function corrupted with noise. If $|\alpha|=1$, then by our assumption on \sr{checkCor}, the target linear function must be found in line~\ref{line:LDME: checkCor1}. Therefore, we assume that $2\le|\alpha|\le k$. In this case, we show that the reduced binary instance is the one of LBP with the correlation $\frac{\rho}{2q^3}$ w.h.p. We assume that, as mentioned in the definition of the algorithm \sr{main2}, all columns are labeled by vectors in $\F_q^n$. In addition, assume that the algorithm \sr{main2} succeeds in selecting $(J,\bar{J}),a_1,a_2,s_1$, and $s_2$ satisfying that
\begin{itemize}
	\item $1\le|\alpha^J| \le \lceil k/2 \rceil$ and $1\le|\alpha^{\bar{J}}|\le \lfloor k/2 \rfloor$ (by Lemma~\ref{lemma:concecutive partitions}, such a consecutive partition must exist)
	\item $\Pr_{x,f}[f(x)-\chi_{\alpha^{J}}-a_1 = \chi_{\alpha^{\bar{J}}} = a_2 ]\ge 1/q+\rho/q^3$ (by Claim~\ref{claim:correlated}, such values of $a_1,a_2$ must exist)
	\item $init(\alpha^J) = s_1$ and $init(\alpha^{\bar{J}}) = s_2$
\end{itemize}
Then, the reduced instance must contain the pair of columns indexed by $(\alpha^J,\alpha^{\bar{J}})$, we call it the target pair. For any pair of columns except for the target pair, as mentioned in the observation following Claim~\ref{claim:case1}, the pair in the reduced instance is  uniformly and independently distributed over $\{\pm 1\}^d$. On the other hand, for each row of the target pair, their product is also $\{\pm1\}$-valued and the expectation is at least $\rho/q^3$ by Claim~\ref{claim:correlation remain in the target}. If we select the sample size $d$ to be more than $\frac{8q^6}{\rho^2}\ln 4$, then by Hoeffding inequality, the probability that their inner product does not exceed $\rho/2q^3$ is bounded above by
\[\exp\left(-\frac{2\rho^2d}{4\cdot 4q^6}\right)\le \exp\left(-\frac{\rho^2}{8q^6}\frac{8q^6}{\rho^2}\ln 4\right) = \frac{1}{4}.\] 
In other words, with probability at least $3/4$, the algorithm reduces LDME to LBP of the correlation $\rho/2q^3$. W.l.o.g., we can assume that the failure probability of \sr{LBP} is at most $1/4$, (otherwise, it is achieved by constant number of repetitions). Thus, for each trial in lines~\ref{line:LDME: make instance} and~\ref{line:LDME execution of LBP}, the probability that \sr{LBP} does not find the target pair is at most $1/2$. Therefore, by repeating these trials at least $\log{2/\delta}$ times, the failure probability decreases to $\delta/2$. Even if we consider the possibility that \sr{checkCor} may fail, the total failure probability is bounded above by $\delta/2+\delta/2 = \delta$. The total running time is bounded above by
\begin{multline*}
nq\cdot \poly(n,\rho^{-1},\ln{\delta^{-1}}) + O(nq^4\cdot\ln{\delta^{-1}}\cdot dq^{k/2}n^{k/2})(T((qn)^{k/2},\rho/2q^3) + \poly(n,\rho^{-1},\ln{\delta^{-1}})) \\
\le \poly(n,\rho^{-1},\ln{\delta^{-1}}) \cdot d\cdot T((qn)^{k/2},\rho/2q^3).
\end{multline*}
\end{proof}

\end{document}